\documentclass[10pt]{article}
\usepackage[margin=0.75in]{geometry}

\usepackage[pdftex]{graphicx}
\graphicspath{{../pdf/}{../jpeg/}}
\DeclareGraphicsExtensions{.pdf,.jpeg,.png}
\usepackage{amssymb,amsmath,url}
\usepackage{bm} 
\usepackage{multirow}
\usepackage{booktabs}	
\usepackage{color}
\usepackage{subfigure}
\usepackage{bigfoot}		
\usepackage{algpseudocode}
\usepackage{pifont}
\usepackage{tikz}
\usepackage{color}
\usepackage{cite}
\usepackage{setspace}
\usepackage{lipsum}
\usepackage{algorithm} 
\usepackage{algpseudocode}
\usepackage[pdftex]{graphicx}
\usepackage{epsfig}
\usepackage{comment}
\usepackage{enumitem} 
\usepackage{amssymb}
\usepackage{hyperref}
\usepackage{fontawesome5}

\algrenewcommand\algorithmicrequire{\textbf{Input:}}
\algrenewcommand\algorithmicensure{\textbf{Output:}}
\algrenewcommand\algorithmicforall{\textbf{For}}

\usepackage{amsthm} 
  
\newtheorem{proposition}{Proposition}

\allowdisplaybreaks

\newtheorem{remark}{Remark}

\newtheorem{assumption}{Assumption}

\newcommand{\E}{\mathbb{E}}

\DeclareMathOperator*{\argmin}{arg\,min}
\DeclareMathOperator*{\argmax}{arg\,max}

\newcommand{\be}{\begin{equation}}
\newcommand{\ee}{\end{equation}}
\newcommand{\bee}{\begin{eqnarray}}
\newcommand{\eee}{\end{eqnarray}}

\allowdisplaybreaks

\newcommand{\bse}{\begin{subequations}}
	\newcommand{\ese}{\end{subequations}}

\renewcommand{\Pr}{\mathop{\mathbb{P}}}

\newcommand{\defeq}{\overset{\text{\tiny def}}{=}}

\ifodd 0
    \newcommand\rev[1]{{\color{magenta}#1}}
    \newcommand{\com}[1]{\textbf{\color{red} (COMMENT: #1)}} 
\else
    \newcommand\rev[1]{{#1}}
    \newcommand{\com}[1]{}
\fi

\ifodd 0 

\else
    
    \newcommand{\accom}[1]{}
\fi

\ifodd 0 
\newcommand{\cnjCom}[1]{{\textbf{\color{red} (CNJ$\to$ #1)}}}
\else
    
    \newcommand{\cnjCom}[1]{}
\fi


\title{\LARGE \bf 
VABO: Violation-Aware Bayesian Optimization for Closed-Loop Control Performance Optimization with Unmodeled Constraints	
}
\author{Wenjie Xu\thanks{{Laboratoire d’Automatique}, École polytechnique fédérale de Lausanne, Lausanne, Switzerland.~{\faIcon{envelope}~\texttt{\{wenjie.xu, colin.jones\}@epfl.ch}}.} \footnotemark[2],   
Colin N Jones\footnotemark[1],  Bratislav Svetozarevic\thanks{Swiss Federal Laboratories for Materials Science and Technology, Switzerland.~\faIcon{envelope}~\texttt{bratislav.svetozarevic@empa.ch}},\\  Christopher R. Laughman\thanks{Mitsubishi Electric Research Laboratories, Cambridge, MA, USA.~\faIcon{envelope}~\texttt{laughman@merl.com}}, 
Ankush Chakrabarty\footnotemark[3] \footnotemark[4]\thanks{Corresponding author. \faIcon{envelope}~\texttt{achakrabarty@ieee.org}. \faIcon{phone-square-alt}~+1 (617) 758-6175. \faIcon{map-marker-alt} 201 Broadway, 8th Floor, Cambridge, MA 02139, USA.}
\thanks{This work has received support from the Swiss National Science Foundation under the NCCR Automation project, grant agreement 51NF40\_180545.}}
\allowdisplaybreaks
\begin{document}
\maketitle 
\begin{abstract}\noindent We study the problem of  performance optimization of closed-loop control systems with unmodeled dynamics.  Bayesian optimization (BO) has been demonstrated effective for improving closed-loop performance by automatically tuning controller gains or reference setpoints in a model-free manner. However, BO methods have rarely been tested on dynamical systems with unmodeled constraints. In this paper, we propose a violation-aware BO algorithm (VABO) that optimizes closed-loop performance while simultaneously learning constraint-feasible solutions. Unlike classical constrained BO methods which allow an unlimited constraint violations, or `safe' BO algorithms that are conservative and try to operate with near-zero violations, we allow budgeted constraint violations to improve constraint learning and accelerate optimization. We demonstrate the effectiveness of our proposed VABO method for energy minimization of industrial vapor compression systems. 
\end{abstract}
\vspace{.1cm}
\section{INTRODUCTION}
\vspace{.1cm}
Closed-loop systems can often be optimized after deployment by altering controller gains or reference inputs guided by the performance observed through operational data. Manually tuning these control parameters often requires care and effort along with considerable task-specific expertise. Algorithms that can automatically adjust these control parameters to achieve optimal performance are therefore invaluable for saving manual effort, time, and expenditure. 

Optimal performance of a control system is generally defined via domain-specific performance functions whose arguments are outputs that can be measured from the closed-loop system. While the map from measurements to performance may be clearly defined, the map from control parameters (that can actually be tuned) to performance is often unmodeled or unknown, since closed-form system dynamics may not be available during tuning~\cite{chakrabarty2021_VCS}. It is thus common to treat the control parameters-to-performance map as a black-box, and design the tuning algorithm to perform in a data-driven manner, where data is collected by performing experiments or simulating a model. However, since both experimentation and high-fidelity software simulations are expensive, tuning algorithms must be designed to assign a near-optimal set of control parameters with as few experiments/simulations (equivalently, performance function evaluations) as possible.

It is precisely for this reason that Bayesian optimization~(BO) has received widespread attention in the context of closed-loop performance optimization. BO is a sample-efficient derivative-free global optimization method~\cite{jones1998efficient, frazier2018tutorial} that utilizes probabilistic machine learning to intelligently search through high-dimensional parameter spaces.
In recent work, BO has demonstrated potential in controller gain tuning~\cite{lederer2020parameter, duivenvoorden2017constrained,khosravi2019controller, konig2020safety}, MPC tuning~\cite{bansal2017goal, piga2019performance, paulson2020data} and in various real-world control applications, such as wind energy systems~\cite{baheri2017altitude, baheri2020waypoint}, engines~\cite{pal2020multi} and space cooling~\cite{chakrabarty2021_VCS}.

A challenge that has garnered recent interest is that of \textit{safe} Bayesian optimization; that is, BO in the presence of safety-critical system constraints. These constraints may also be unmodeled (`black-box'), as a mathematical representation of the constraint with respect to the control parameters is not always known or straightforward to represent. To handle these constraints, a variety of so-called `safe BO' techniques have been recently proposed; these methods either operate on the principle of minimizing constraint violations during optimization~\cite{sui2015safe, sui2018stagewise, turchetta2020safe}, or leverage partial model knowledge to ensure safety via Lyapunov arguments~\cite{chakrabarty2021safe}. In either case, safe BO achieves the goal of learning feasible optima without violating unmodeled constraints, or risking their violation with predefined small probability. Often, this conservativeness results in obtaining local minima, slow convergence speeds, and reduced data-efficiency.
Conversely, generic constrained optimization with BO learns constraints without paying heed to the amount of constraint violation during the exploration phase. These methods~\cite{gardner2014bayesian, gelbart2014bayesian}, are mostly agnostic to the deleterious consequences of constraint violation, such as long-term damage to expensive hardware caused by large violation, rendering them impractical for many industrial applications. A more violation-aware direction of BO research proposes the use of budgets on the cost of constraint violation~\cite{snoek2012practical, lee2020cost, lam2016bayesian, marco2020excursion}. However, the budget considered~(e.g., neural network training time~\cite{snoek2012practical}) is usually related to the effort or failure risk for performance function evaluation, and does not provide a way to manage the magnitude of constraint violations. 

For many industrial systems,  small constraint violations over a short period are acceptable if that exploration results in speedup of the optimization procedure, but large violations are heavily discouraged. Concretely, we study vapor compression systems (VCSs) where it is imperative that constraint violation on system variables such as compressor discharge temperature are limited to short time periods, and therefore, the interplay of informativeness and safety can be systematically adjusted through budgeting. In other words, for VCS, the benefits of accelerated global convergence significantly outweigh the cost of short-term constraint violations. 

In this paper, we propose a novel violation-aware Bayesian optimization~(VABO) that exhibits accelerated convergence compared to safe BO algorithms, while ensuring the violation cost is within a prescribed budget. We demonstrate that our proposed VABO algorithm is less conservative than `safe BO' algorithms that tend to have sample-inefficient exploration phases and can get stuck in a local minimum because they cannot allow any constraint violation. The VABO algorithm is also more cautious than constrained Bayesian optimization, which is agnostic to constraint violations and thus, is likely to incur significant violation costs. Our VABO algorithm is based on the principle of encouraging performance function evaluation at combinations of control parameters that greatly assist the optimization process, as long as it does not incur high constraint violations likely to result in system failure or irreversible damage.    

Our \textbf{contributions} include:

\begin{enumerate}
\item we propose a variant of  constrained BO methods for control parameter tuning that improves global convergence rates within a prescribed amount of constraint violation with guaranteed high probability;

\item we propose a simple and tractable constrained auxiliary acquisition function optimization problem for trading off performance improvement and constraint violation; and,

\item we validate our algorithms on a set-point optimization problem using a high-fidelity VCS that has been calibrated on an industrial HVAC system.
\end{enumerate}

\section{Preliminaries}
\subsection{Problem Statement}
We consider closed-loop systems of the form
\bee\label{eq:cl-sys}
\xi_+ = F(\xi, \theta),
\eee
where $\xi,\xi_+\in\mathbb R^{n_\xi}$ denote the system state and update (respectively), $\theta\in\Theta\subset \mathbb R^{n_\theta}$ the control parameters (e.g., set-points) to be tuned, and $F(\cdot,\cdot)$ the closed-loop dynamics with initial condition $\xi_0$. We assume that the closed-loop system~\eqref{eq:cl-sys} is designed such that it is exponentially stable to a control parameter dependent equilibrium state $\xi^\infty(\theta)$ for every $\theta\in\Theta$. We further assume that $\xi^\infty(\cdot)$ is a continuous map on $\Theta$. For simplicity, we also assume $\Theta$ is a box (product of intervals) in $\mathbb R^{n_\theta}$.

To determine the system performance, we define a continuous cost function $\ell(\theta):\mathbb{R}^{n_\theta}\mapsto\mathbb{R}$ to be minimized, which is an unknown/unmodeled function of the parameters $\theta$. This is not unusual: while $\ell$ may be well-defined in terms of system outputs, it is often the case that the map from control parameters to cost remains unmodeled; in fact, $\ell$ may not even admit a closed-form representation on $\Theta$, c.f.~\cite{chakrabarty2021_VCS,burns2018proportional}.

We also define $N$ unmodeled constraints on the system outputs that require caution during tuning. \rev{The $i$-th such constraint is given by $g_i(\theta):\mathbb{R}^{n_\theta}\mapsto\mathbb{R},i\in[N]$}, where the notation $[N]\defeq\{i\in\mathbb{N}, 1\leq i\leq N\}$; we assume each $g_i(\cdot)$ is continuous on $\Theta$. We assume that the cost function $l(\theta)$ and every constraint $g_i(\theta), \; i\in[N]$ can be ascertained, either by measurement or estimation, during the hardware/simulation experiment. We introduce the brief notation $g(\cdot)\le 0\triangleq g_i(\cdot)\le 0,\; i\in [N]$ and assume that an initial feasible set of solutions is available at design time.
\begin{assumption}\label{asmp:initial_set}
The designer has access to a non-empty safe set $\Theta_0$ such that for any $\theta\in\Theta_0$, no constraint is violated; that is, $g(\theta)\le 0$ for every $\theta\in\Theta_0$.
\end{assumption}
While such an initial set $\Theta_0$ can be derived using domain expertise, it is likely that $\Theta_0$ contains only a few feasible $\theta$, and at worst could even be a singleton set. 

We cast the control parameter tuning problem as a black-box constrained optimization problem, formally described by 
\bse\label{eqn:formulation}
\begin{align}
\min_{\theta\in\Theta}  \qquad& \ell(\theta),\label{eqn:obj} \\
\text{subject to:} \qquad & g_i(\theta)\leq 0,\quad \forall i\in[N].\label{eqn:constraint}
\end{align}
\ese
Our \textbf{objective} is to solve the constrained optimization problem~\eqref{eqn:formulation} with limited constraint violations \rev{during the optimization process}. 
Since the constraints are assumed to be unmodeled and a \rev{limited set} of feasible solutions is known at design time, we \rev{do not expect a guarantee of zero constraint violations. The tolerable amount and duration of constraint violations are problem-dependent.} \com{Bratislav \& Wenjie: Not actually, it depends on the algorithm; e.g. SafeBO will not violate; If you sample conservatively, the sampling efficiency is bad. If we strategically sample, as in VABO, we may have some limited, but controlled constraint violation. This limit we formally introduce as a \textit{budget for allowed constraint violations}.  Why not exploit constraint violation to obtain better solutions / exploring better solutions. Try to address the question of encountering the associated risks of allowing the algorithm to violate the constraints slightly -- within the budget. But this needs to be allowed by the domain expert.}\rev{In some applications, such as vapor compression systems, small constraint violations over a short-term are acceptable, while large constraint violations are strongly discouraged.} \rev{In such cases,} instead of being overly cautious and ending up with suboptimal solutions, we allow small constraint violations \textit{as long as the resulting knowledge gathered by evaluating an infeasible (in terms of constraint violation) $\theta$ accelerates the optimization process or helps avoid local minima}.
\cnjCom{This paragraph seems too strong. The formulation we take does not necessarily penalize large violation like this. We just assume that there is a user-provided function that measures how bad a violation is. It could well be the case that the engineer sets it up so that a large, but short, violation is considered better than a short, but long, violation (e.g., as is very common when tuning batch controller for process control, when off-spec means that product must be thrown away - so large and small violations are the same, only violation time matters)}\com{Wenjie: I think our problem formulation can take into account the user-specific tolerance-level to different patterns of constraint violations. I agree that our formulation does not necessarily requires penalizes large violations a lot. But I think in many application scenarios, it's true~(e.g., VCS). } \com{Bratislav: We can address those differences in large and short violations and small and long, but, not in this formulation. It would be possible, but we need to introduce time dependent violation, and this we shall do in the journal paper, and here say for future work.}

\begin{remark}
Our formulation~\eqref{eqn:formulation} can also optimize batch processes over finite time-horizons, say $T_h$. This would involve defining the objective and constraints over a batch trajectory with stage loss $\ell(\theta):=\tfrac{1}{T_h} \int_0^{T_h} l(\tau,\theta) \,\mathrm{d}\tau$.
\end{remark}

\subsection{Proposed Solution}
We propose a modified Bayesian optimization framework to solve the problem~\eqref{eqn:formulation} that is violation-aware: the algorithm automatically updates the degree of risk-taking in the current iteration based on the severity of constraint violations in prior iterations. Concretely, for an infeasible  $\theta$, the constraint violation cost is given by  
\bee
\bar c_i(\theta) \triangleq c_i\left([g_i(\theta)]^+\right)\label{equ:violation_cost}, \quad i\in [N]
\eee
where $[g_i(\cdot)]^+ := \max\{g_i(\cdot), 0\}$ and $c_i:\mathbb{R}_{\ge 0}\mapsto\mathbb{R}_{\ge 0}$. \rev{Note that $g_i$ corresponds to physically meaningful system outputs that we can measure, e.g., temperature.}
This violation cost function $c_i$ is user-defined as a means to explicitly weight the severity of `small' versus `large' constraint violations. While the function $c_i$ is at the discretion of the designer, it needs to satisfy the following mild assumptions in order to achieve desirable theoretical properties; see \S3.
\begin{assumption}
\label{assump:cost_func}
The violation function $c_i$ satisfies:
\begin{enumerate}[label=(A\arabic*)]
\item $c_i(0)=0$,\label{enu:zero_vio_zero_cost}
\item $c_i(s_1)\rev{\geq} c_i(s_2)$, if $s_1>s_2\geq 0$,\label{enu:mono}
\item $c_i$ is left continuous on $\mathbb{R}_{\ge0}$.
\end{enumerate}
\end{assumption}

Assumption~\ref{assump:cost_func} captures some intuitive properties required of the violation function. According to (A1), there is no cost associated with no violation. From (A2), we ensure that the violation cost is monotonically \rev{non-decreasing} with increased violations. Finally (A3) ensures that this monotonic increase is smooth and does not exhibit discontinuous jumps. 

\cnjCom{It's not clear that we need this function $c$, or that Assumption 2 does anything / is valid.\\[5pt]
The function $g$ is entirely user-defined. This means that it could be set to be $g_i(\theta) = c_i^{-1} \circ \hat g_i(\theta)$, which means that we can "undo" the action of $c_i$ through the selection of $g$, and thereby have $\bar c_i$ be anything at all. i.e., the assumption 2 can be met, and yet the map from $\theta$ to $\bar c$ can be anything.\\[5pt]
Because there's enough degrees of freedom in the selection of $g$ - we could just replace $c$ by $[g]^+$.\\[5pt]
All the properties listed in the paragraph above are also satisfied by $c = [g]^+$}
\com{Wenjie:Mathematically maybe true. But I think here $g$ has special physical meaning in different applications, like the tempratures, etc.. And $c_i$ is input from user, who may have a clear physical understanding how much to penalize different amount of violations. If we fix $c$ to be $[g]^+$, then we move the complexity to the design of $g$, which may not be straightforward for a user of our algorithm.}

To adapt the degree of risk-taking based on prior data obtained, we define a violation budget over a horizon of $T\in\mathbb N$ optimization iterations.
Our proposed VABO algorithm is designed to sequentially search over $T$ iterations $\{\theta_t\}_{t=1}^T$ while \rev{using a prescribed budget} of constraint violations in order to obtain a constraint-optimal set of parameters
\bee\label{eq:budgeted_optimization_problem}
\notag &\theta^\star_T \;:=\; \argmin_{\substack{t\in[T];\\ g(\theta_t)\leq 0}} \ell(\theta_t)\;\; \textrm{subject to:}\;\;
\sum_{t=1}^{T}\bar c_i(\theta_t)\leq B_i,
\eee
where $B_i$ denotes a budget allowed for the $i$-th violation cost. 

Note that this formulation is a generalization of  well-known constrained Bayesian optimization formulations proposed in the literature. For instance, if we set $B_i=0,\forall i\in[N]$, then our  formulation is reminiscent of safe Bayesian optimization~\cite{sui2015safe,sui2018stagewise}. Alternatively, setting $B_i\equiv \infty$ reduces our problem to constrained Bayesian optimization problem without violation cost consideration~\cite{gardner2014bayesian, gelbart2014bayesian}.

\section{Violation-Aware Bayesian Optimization (VABO)}
\subsection{\rev{Bayesian Optimization Preliminary}}
For Bayesian optimization, one models $\ell(x)$ and $g(x)$ as functions sampled from independent Gaussian processes. At iteration $t$, conditioned on previous function evaluation data $\mathcal{D} := \left\{\theta_{1:t}, \ell(\theta_{1:t})\right\}$, the posterior mean and standard deviation of $\ell$ is given by 
$
 \mu_\ell(\theta|\mathcal{D}) = k_\ell^\top(\theta, \theta_{\mathcal{D}}) K_\ell^{-1}\Delta y_\ell + \mu_{\ell,0}(\theta)$ and 
$\sigma^2_\ell(\theta|\mathcal{D})=k_\ell(\theta, \theta)-k_\ell^\top(x, \theta_{\mathcal{D}})K_\ell^{-1}k_\ell(\theta_{\mathcal{D}}, \theta)$,
where $\theta_{\mathcal{D}}=\theta_{1:t}$ is the set of control parameters with which previous experiments/simulations have been performed. Here,
\begin{align*}
&k_\ell(\theta,\theta_{\mathcal{D}}) \triangleq[k_\ell(\theta,\theta_i)]_{\theta_i\in \theta_{\mathcal{D}}},\;
k_\ell(\theta_{\mathcal{D}},\theta) \triangleq[k_\ell(\theta_i,\theta)]_{\theta_i\in \theta_{\mathcal{D}}},\\
& K_\ell \triangleq \left(k_\ell(\theta_{i},\theta_{j})\right)_{\theta_i,\theta_j\in \theta_\mathcal{D}},\;
\Delta y_\ell \triangleq[\ell(\theta_i)-\mu_{\ell,0}(\theta_i)]_{\theta_i\in \theta_\mathcal{D}},
\end{align*}
and $k_\ell(\cdot, \cdot)$ is a user-defined kernel function and $\mu_{\ell,0}$ is the prior mean function, both associated with $\ell$; see~\cite{frazier2018tutorial} for more details on kernel and prior mean selection. The above quantities are all assumed to be column vectors, except $K_\ell$, which is a positive-definite matrix.
For the constraint functions $g$, similar expressions for the \rev{posterior} mean $\mu_{g_i}(\theta|\mathcal{D})$ and \rev{standard deviation} $\sigma_{g_i}(\theta|\mathcal{D})$ can be obtained.

The kernelized functions above provide tractable approximations of the performance cost of the closed-loop system, along with the constraint functions, both of which were hitherto unmodeled/unknown. Classical BO methods operate by using the statistical information embedded within these approximations to intelligently explore the search space $\Theta$ via acquisition functions. A specific instance of an acquisition function commonly used in constrained BO is the constrained expected improved (CEI) function~\cite{gardner2014bayesian}, given by
\begin{equation}
\mathsf{CEI}(\theta|\mathcal{D})=\E\left(\prod_{i\in[N]}\mathbf{1}_{g_i(\theta)\leq 0}\; I(\theta)|\mathcal{D}\right).\label{equ:eic_def}
\end{equation}
where $\mathbf{1}$ denotes the indicator function, $\mathbb E$ denotes the expectation operator, and
$
I(\theta)=\max\{0, f(\theta^{\min}_t)-f(\theta)\}$
is the improvement of $\theta$ over the incumbent solution over $t$ iterations. Here, the incumbent \rev{best} solution is given by
\bee\label{eq:incumbent}
\theta_t^{\min}=\argmin_{\{\theta_\tau|\tau\in[t-1], g_i(\theta_\tau)\leq 0, \forall i\in[N]\}\, \cup\, \Theta_0} \ell(\theta),
\eee
with $\Theta_0$ being the initial safe set of feasible points.

As $g_i(\theta),\forall i\in[N]$ and $\ell(\theta)$ are independent, we deduce
\begin{equation}\notag
\mathsf{CEI}(\theta|\mathcal{D})=\prod_{i\in[N]} \Pr(g_i(\theta)\leq 0|\mathcal{D})\E\left(I(\theta)|\mathcal{D}\right).
\end{equation}
This acquisition function is therefore modulated by the feasibility probability
\bse
\bee\label{eq:acqfn_cei_a}
\Pr(g_i(\theta)\leq0|\mathcal{D})=\Phi\left(\tfrac{-\mu_{g_i}(\theta|\mathcal{D})}{\sigma_{g_i}(\theta|\mathcal{D})}\right),
\eee
where $\Phi(\cdot)$ is the cumulative distribution function of standard normal distribution. The closed-form expression of unconstrained expected improvement is derived in~\cite{jones1998efficient},
\bee\label{eq:acqfn_cei_b}
\E\left(I(\theta)|\mathcal{D}\right)=\left(\ell(\theta_t^{\min})-\mu_{\ell}(\theta|\mathcal{D})\right) \Phi\left(z\right)+\sigma_\ell(\theta|\mathcal{D})\phi\left(z\right),
\eee
\ese
where $\phi(\cdot)$ is the standard normal probability density function and 
$z=\tfrac{\ell(\theta_t^{\min})-\mu_\ell(\theta|\mathcal{D})}{\sigma_\ell(\theta|\mathcal{D})}$. Combining~\eqref{eq:acqfn_cei_a} and~\eqref{eq:acqfn_cei_b} yields a tractable expression for computing \rev{$\mathsf{CEI}$ in}~\eqref{eqn:acquisition_problem}.

\subsection{The VABO Algorithm}

Our VABO algorithm proposes an auxiliary optimization problem that leverages the constrained expected improvement acquisition function to select feasible good candidates while ensuring (with high probability) that the violation cost of evaluating this candidate will remain within a prescribed budget.

Let the remaining budget at the $t$-th VABO iteration be 
$B_{i,t} \triangleq B_i-\sum_{\tau=1}^{t-1}\bar c_i(\theta_{\tau})$.
At this iteration, we solve the following \textit{auxiliary} problem \bse\label{eqn:acquisition_problem}
\begin{align}
&\theta^\star_t := \argmax\limits_{\theta\in\Theta}\;\mathsf{CEI}(\theta|\mathcal{D}),\label{eqn:acq_obj} \\
\text{subject to:}\qquad & \prod_{i\in[N]}\Pr(\bar c_i(\theta)\leq \beta_{i,t}B_{i,t})\geq 1-\epsilon_t, \label{eqn:acq_constraint}
\end{align}
\ese
to compute the next control parameter candidate $\theta_t^\star$, where $\beta_{i,t}\in[0,1]$ is a user-defined weighting scalar that determines how much of the remaining budget can be used up, and $0<\epsilon_t\ll 1$
determines the probability of large constraint violation. Note that~\eqref{eqn:acq_obj} involves maximizing the constrained expected improvement objective, which is common to cBO algorithms; c.f.~\cite{gardner2014bayesian}. Our modification using the budget, as written in~\eqref{eqn:acq_constraint}, enforces that the next sampled point will not use up more than a $\beta_{i,t}$ fraction of the remaining violation cost budget $B_{i,t}$ for all constraints with a probability of at least $1-\epsilon_t$, conditioned on the data seen so far. \rev{This modification allows us to trade a prescribed level of violation risk for more aggresive exploration, leading to faster convergence.} Although we are using \rev{the commonly used} constrained expected improvement acquisition function here, we can easily generalize to other acquisition functions by simply replacing $\mathsf{CEI}$ in the objective~\eqref{eqn:acq_obj}.

We now discuss how to efficiently solve the auxiliary problem~\eqref{eqn:acquisition_problem}. Recall from Assumption~\ref{assump:cost_func} that $c_i$ is non-decreasing on $\mathbb{R}_{\ge 0}$ for every $i\in[N]$. Therefore, we can define an inverse violation function
$c_i^{-1}(s) = \sup\{r\in\mathbb{R}_{\ge 0} \mid c_i(r)\leq s\}$
for any $s\in\mathbb{R}_{\ge 0}$.
Therefore, we can write
\begin{align*}
\Pr(\bar c_i(\theta)\leq\beta_{i,t}B_{i,t}|\mathcal{D})&=\Pr([g_i(\theta)]^+\leq c_i^{-1}(\beta_{i,t}B_{i,t})|\mathcal{D})\\
&=\Pr(g_i(\theta)\leq c_i^{-1}(\beta_{i,t}B_{i,t})|\mathcal{D}).
\end{align*}
Since $g_i(\theta)$ follows a Gaussian distribution with mean $\mu_{g_i}(\theta|\mathcal{D})$ and variance $\sigma_{g_i}^2(\theta|\mathcal{D})$, we get
\[
\Pr(\bar c_i(\theta)\leq\beta_{i,t}B_{i,t}|\mathcal{D})=\Phi\left(\frac{c_i^{-1}(\beta_{i,t}B_{i,t})-\mu_{g_i}(\theta|\mathcal{D})}{\sigma_{g_i}(\theta|\mathcal{D})}\right).
\]
When the number of control parameters $n_\theta$ 
is small ($<6$), we can place a grid on $\Theta$ and evaluate the cost and constraints of~\eqref{eqn:acquisition_problem} at all the grid nodes. The maximum feasible solution can then be used as the solution to the auxiliary problem. When $n_\theta>6$, we can use gradient-based methods with multiple starting points to solve problem~\eqref{eqn:acquisition_problem}, since evaluating the learned GPs approximating $\ell$ and $g$ require very little computational time or effort when $T$ 
is not large (empirically, $<2000$).
The VABO algorithm is terminated either when $T$ iterations of the algorithm have been reached, or the cumulative violation cost for constraint $i$ exceeds $B_i$ for any $i\in [N]$. 
We provide pseudocode for implementation in Algorithm~\ref{alg:VABO}.

The following proposition provides a probabilistic guarantee of using up the given number of samples while keeping the violation cost below the given budget. It highlights the ``violation-awareness'' property exhibited by VABO.
\begin{proposition}\label{thm:vio_aware}
Fix $\delta\in (0,1)$ and  $T\in\mathbb N$. For any $\beta_{i,t}\in (0,1]$, if $\epsilon_t$ are chosen for every $t\in [T]$ such that $\delta = 1-\prod_{t=1}^T (1-\epsilon_t)$, then the VABO algorithm described in Algorithm~\ref{alg:VABO} satisfies the inequality
\[
\Pr\left(\left\{T\text{ iters are used}\right\}\cap\left\{\sum_{t=1}^{T}\bar c_i(\theta_t)\leq B_i,\forall i\in[N]\right\}\right)\geq 1-\delta.
\]
\end{proposition}
\begin{proof}
\small Let $$\mathcal{E}_t:=\{t\text{ iters are used}\}\cap\left\{\sum_{\tau=1}^{t}\bar c_i(\theta_\tau)\leq B_i,\;\forall i\in[N]\right\},$$
where $t\in[T]$.
We have 
\begin{align}
\Pr\left(\mathcal{E}_T\right)&=\Pr\left(\mathcal{E}_{T-1}\right)\Pr\left(\mathcal{E}_{T}\mid \mathcal{E}_{T-1}\right)\nonumber\\
&=\Pr\left(\mathcal{E}_{T-1}\right)\Pr\left(\bar c_i(\theta_T))\leq B_{i,T},\forall i\in[N]\mid\mathcal{E}_{T-1}\right)\nonumber\\
&\geq\Pr\left(\mathcal{E}_{T-1}\right)\Pr\left(\bar c_i(\theta)\leq\beta_{i,T}B_{i,T},\forall i\in[N]\mid\mathcal{E}_{T-1}\right)\nonumber\\
&\geq \Pr\left(\mathcal{E}_{T-1}\right)(1-\epsilon_T).\nonumber
\end{align}

By recursion, we have
$\Pr\left(\mathcal{E}_T\right)
\geq\Pr\left(\mathcal{E}_1\right)\prod_{t=2}^{T}(1-\epsilon_t)\geq\prod_{t=1}^{T}(1-\epsilon_t)\nonumber=1-\delta$, which concludes the proof.
\end{proof}

\begin{algorithm}[htbp!]
	\caption{Violation-Aware Bayesian Optimization}\label{alg:VABO}
	\begin{algorithmic}[1]
	\normalsize
	\State \textbf{Require}: Max VABO iterations $T$, violation budget $B_i,\forall i\in[N]$ and an initial safe set of points $\Theta_0$ 
	\State Evaluate $\ell(\theta)$, $g_i(\theta),\forall i\in[N]$ for $\theta\in \Theta_0$ by performing experiments or simulation
	\State Calculate incumbent solution using~\eqref{eq:incumbent} \State Initialize dataset $\mathcal D=\{\theta_t, \ell(\theta_t), g(\theta_t)\}\; \forall\theta_t\in\Theta_0$
	\For{$t\in[T]$}
	    \State $\Theta_t=\left\{\prod_{i\in[N]}\Pr(\bar c_i(\theta)\leq \beta_{i,t}B_{i,t}|\mathcal{D})\geq 1-\epsilon_t|\theta\in\Theta\right\}$
    \State $\theta_{t}^\star=\argmax_{\theta\in \Theta_t}\mathsf{CEI}(\theta|\mathcal{D})$ \Comment{Solving~\eqref{eqn:acquisition_problem}}
	    \State $\ell(\theta_t^\star), g(\theta_t^\star)\leftarrow$  perform experiment with $\theta_t^\star$
	    \State $B_{i,t+1}=B_{i,t}-\bar c_i(\theta_t),\forall i\in[N]$
	    \If{$\exists\, i\in[N]$, such that $B_{i,t+1}<0$}
	    \State \textbf{return} $\theta_{t+1}^{\min}$
	    \EndIf
	    \State Update dataset, $\mathcal D\leftarrow \mathcal D \cup \{\theta_t^\star, 
	    \ell(\theta_t^\star), g(\theta_t^\star)\}$ 
	    \State Update Gaussian process posterior 
	\EndFor
	\State \textbf{return} VABO solution: $\theta_{T+1}^{\min}$
	\end{algorithmic}
\end{algorithm}

\section{Case Study: VCS Optimization with  Temperature Constraints} 

In this section, we consider the problem of safely tuning set-points of a vapor compression system~(VCS). As shown in Fig.~\ref{fig:vcs_vabo_diagram}, a VCS typically consists of a compressor, a condenser, an  expansion valve, and an evaporator. 
While physics-based models of these systems can be formulated as large sets of nonlinear differential algebraic equations to predict electrical power consumption, there are a variety of challenges in developing and calibrating these models. This motivates interest in directly using measurements of the VCS power under different operating conditions, and assigning optimal set-points to the VCS actuators using data-driven, black-box optimization methods such as BO, which minimize the power consumption. 

\begin{figure}[!ht]
    \centering
    \includegraphics[width=0.6\columnwidth]{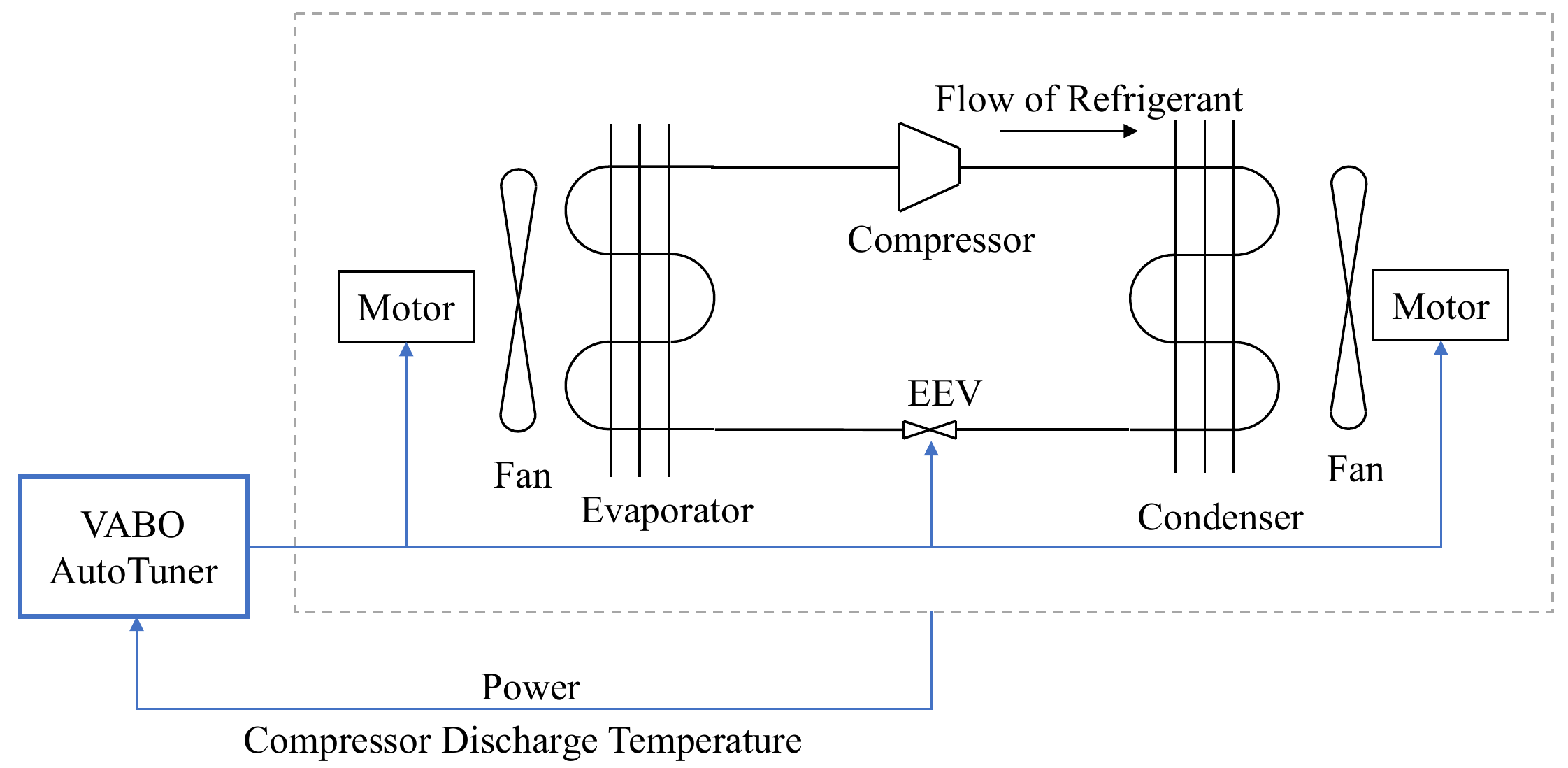}
    \caption{Schematic diagram of vapor compression system with our proposed VABO controlling the EEV~(electronic expansion valve) and the two fans' speed. For simplicity, we do not show other measurements and controls.}
    \label{fig:vcs_vabo_diagram}
\end{figure}

Reckless \rev{alterations} of the set-points can drive the system into operating modes that reduce the reliability or lifespan of the system. To avoid these deleterious effects, one can add several constraints during the tuning process. One such constraint, considered here, is on the temperature of the refrigerant leaving the compressor, also referred to as the `discharge temperature'. 
The discharge temperature must be managed because compressors are designed to operate within specific temperature ranges; excessively high temperatures can result in the breakdown of refrigerant oils and increase wear and tear.  In addition, high temperatures are often correlated with high pressures, which can cause metal fatigue and compromise the integrity of the pressurized refrigerant pipes in extreme cases. 
While managing the constraints mentioned above in the long run is critical, we also observe that small violations over short periods of time have limited harmful effects. Indeed, it may be beneficial to take the risk of short-period limited violation to accelerate the tuning process.

%

We cast the VCS optimization problem in the same form as ~\eqref{eqn:formulation}, with $\ell$ denoting the steady-state power of the VCS with set points $\theta$. The constraint $g\le 0$ is given by $T_d(\theta)-\hat T_d\leq0$, where \rev{$T_d(\theta)$ is the steady-state discharge temperature with set points $\theta$ and} $\hat T_d$ is a safe upper bound. We close a feedback loop from compressor frequency to room temperature, leaving the set of 3 tunable set points $\theta$ as the expansion valve position and the indoor/outdoor fan speeds. The effect of these set points on power and discharge temperature are not easy to model, and no simple closed-form representation exists. In practice, we assign a setpoint $\theta$, wait for an adequate amount of time until the power signal resides within a 95\% settling zone, and use that power value as $\ell(\theta)$ and the corresponding discharge temperature as $T_d(\theta)$.


\subsubsection*{Implementation Details}
For this paper, we use a high-fidelity model of the dynamics of a prototype VCS\footnote{Note that while the behavior of this model have been validated against \rev{a} real VCS, the numerical values and/or performance presented in this work is not representative of any product.} written in the \texttt{Modelica} language~\cite{modelica2017a} to collect data and optimize the set-points on-the-fly. A complete description of the model is available in~\cite{chakrabarty2021_VCS}.
The model was first developed in the \texttt{Dymola}~\cite{dymola2020a} environment, and then exported as a functional mockup unit (FMU)~\cite{fmi2019a}, and its current version comprises 12,114 differential equations. Bayesian optimization is implemented in \texttt{GPy}~\cite{gpy2014}.

We define our set-point search space \rev{$\Theta := [200, 300]\times[300, 400]\times[500, 800]$}, in expansion valve counts, indoor fan rpm, and outdoor fan rpm, respectively. We aim to keep the discharge temperature below $\hat T_d=331$~K; these constraints are set according to domain knowledge~\cite{burns2018proportional}. We initialize the \rev{simulator} at an expansion valve position of 280~counts, an indoor fan speed of 380~rpm, and an outdoor fan speed of 700~rpm, which is known to be a feasible set-point based on experience.
%
Constraint violations are penalized with the function $c_i(s)=s^2,\, s\in\mathbb{R}^+$. The quadratic nature of the violation cost implies that minor violations are not as heavily penalized as larger ones. The reason for this is that small violations over a small period of time are unlikely to prove deleterious to the long-term health of the VCS, whereas large violations could have more significant effects, even over short periods of time; for instance, damage to motor winding insulation or exceeding mechanical limits on the pressure vessel of the compressor. These constraints have been incorporated into the selection of the thresholds $\hat T_d$. Of course, the threshold values and a parameterized violation cost could be considered to be hyperparameters, and hence could be optimized via further experimentation.
We choose the RBF kernel for our problem, which is commonly used in Bayesian optimization~\cite{frazier2018tutorial}, and compare our method to safe BO~\cite{berkenkamp2016safe} and generic cBO~\cite{gelbart2014bayesian}. 

We showcase the effect of varying $B$ by selecting $B=0$, $10$, and $20$, and set $\beta_{t}=\max\{\beta_0, \frac{1}{T-t+1}\}$. This choice allows the algorithm to use an increasing fraction of the budget when approaching the sampling limit $T$. 
A violation of the cost budget of $10$ works well in this case study, and additional violations of the cost budget brings little convergence benefit. We also observe that the impact of $\beta_0$ is insignificant for this example, and therefore we set $\beta_0=1.0$.  We use ``VABO $B$'' to indicate violation aware BO with budget $B$.

\subsubsection*{Results and Discussion}
\begin{figure*}[t]
    \centering
    \includegraphics[width=\columnwidth]{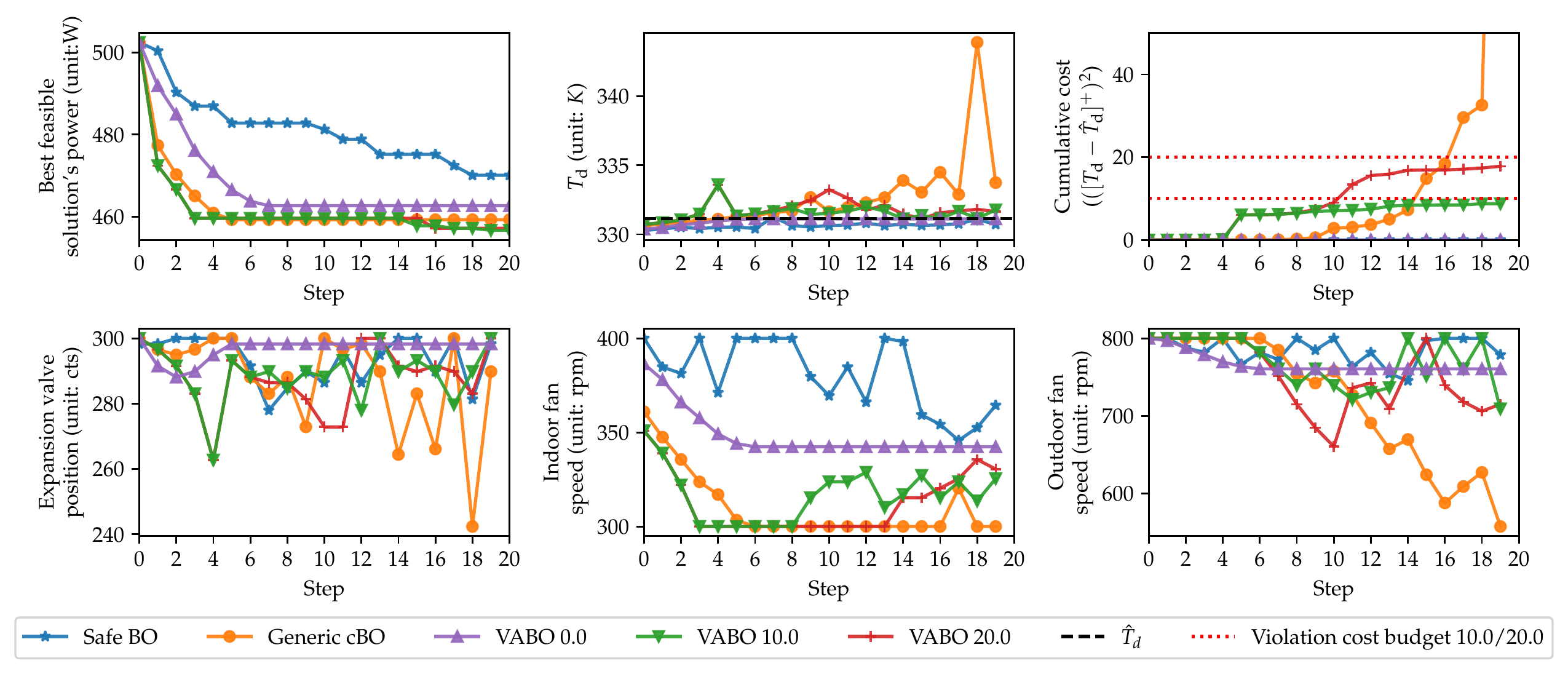}
    \vspace{-2.8ex}
    \caption{Best feasible solution's power, discharge temperature, cumulative violation cost and the three set-points' evolution. }
    \label{fig:sub_plots}
\end{figure*}
Figure~\ref{fig:sub_plots} illustrates the experimental results. We observe that the best feasible solution's power of our {VABO with $B=10$} decreases {slightly faster than} generic constrained BO and significantly faster than safe BO. At the same time, our method manages the violation cost well under the violation cost budget. In comparison, generic constrained BO incurs very large costs in the last few steps, far exceeding the violation cost budget $10$ by more than an order of magnitude.  The improvements from VABO are due to its recognition that the discharge temperature is sensitive to the expansion valve position, but that this position needs to be adjusted in coordination with the fan speeds.  cBO causes large discharge temperatures after step 12 because it makes large adjustments to the expansion valve position while maintaining the indoor fan speed at a low value, a combination that is not penalized during exploration.  VABO with $B=10$ reduces the power by about $9\%$ compared to the most power-efficient initial safe set-points after $20$ steps. We also observe that the $T_d$ constraint is often violated during exploration, and that large violations are entirely possible without violation awareness, as demonstrated by generic cBO. 
Even without any violation cost budget~($B_i=0$), our violation-aware BO finds better solutions faster than safe BO. One explanation for this is that safe BO tends to waste a lot of evaluations to enlarge the safe set, which may lead to slow convergence. Conversely, VABO implicitly encodes the safe set exploration into the acquisition optimization problem~\eqref{eqn:acquisition_problem} and only enlarges the safe set when necessary for optimization, while keeping the violation risk under control. 

\section{Conclusions}
In this paper, we design a sample-efficient and violation-aware Bayesian optimization~(VABO) algorithm to solve the closed-loop control performance optimization problem with unmodeled constraints, by leveraging the fact that small violations over a short period only incur limited costs during the optimization process in many applications such as for vapor-compression cycles. We strategically trade the budgeted violations for faster convergence by solving a tractable auxiliary problem with probabilistic budget constraints at each step. Our experiments on a VCS show that, as compared to existing safe BO and generic constrained BO, our method simultaneously exhibits accelerated \rev{global} convergence and manages the violation cost.
%

\setstretch{1.05}
\bibliographystyle{IEEEtran}
\bibliography{IEEEabrv,root}

\begin{thebibliography}{10}
\providecommand{\url}[1]{#1}
\csname url@samestyle\endcsname
\providecommand{\newblock}{\relax}
\providecommand{\bibinfo}[2]{#2}
\providecommand{\BIBentrySTDinterwordspacing}{\spaceskip=0pt\relax}
\providecommand{\BIBentryALTinterwordstretchfactor}{4}
\providecommand{\BIBentryALTinterwordspacing}{\spaceskip=\fontdimen2\font plus
\BIBentryALTinterwordstretchfactor\fontdimen3\font minus
  \fontdimen4\font\relax}
\providecommand{\BIBforeignlanguage}[2]{{%
\expandafter\ifx\csname l@#1\endcsname\relax
\typeout{** WARNING: IEEEtran.bst: No hyphenation pattern has been}%
\typeout{** loaded for the language `#1'. Using the pattern for}%
\typeout{** the default language instead.}%
\else
\language=\csname l@#1\endcsname
\fi
#2}}
\providecommand{\BIBdecl}{\relax}
\BIBdecl

\bibitem{chakrabarty2021_VCS}
A.~Chakrabarty, C.~Danielson, S.~A. Bortoff, and C.~R. Laughman, ``Accelerating
  self-optimization control of refrigerant cycles with {B}ayesian optimization
  and adaptive moment estimation,'' \emph{Appl. Therm. Eng.}, vol. 197, p.
  117335, 2021.

\bibitem{jones1998efficient}
D.~R. Jones, M.~Schonlau, and W.~J. Welch, ``Efficient global optimization of
  expensive black-box functions,'' \emph{J. Global Optim.}, vol.~13, no.~4, pp.
  455--492, 1998.

\bibitem{frazier2018tutorial}
P.~I. Frazier, ``A tutorial on {B}ayesian optimization,'' \emph{arXiv preprint
  arXiv:1807.02811}, 2018.

\bibitem{lederer2020parameter}
A.~Lederer, A.~Capone, and S.~Hirche, ``Parameter optimization for
  learning-based control of control-affine systems,'' in \emph{Learning for
  Dynamics and Control}.\hskip 1em plus 0.5em minus 0.4em\relax PMLR, 2020, pp.
  465--475.

\bibitem{duivenvoorden2017constrained}
R.~R. Duivenvoorden, F.~Berkenkamp, N.~Carion, A.~Krause, and A.~P. Schoellig,
  ``Constrained {B}ayesian optimization with particle swarms for safe adaptive
  controller tuning,'' \emph{IFAC-PapersOnLine}, vol.~50, no.~1, pp.
  11\,800--11\,807, 2017.

\bibitem{khosravi2019controller}
M.~Khosravi, A.~Eichler, N.~Schmid, R.~S. Smith, and P.~Heer, ``Controller
  tuning by {B}ayesian optimization an application to a heat pump,'' in
  \emph{2019 18th Eur. Control Conf. (ECC)}.\hskip 1em plus 0.5em minus
  0.4em\relax IEEE, 2019, pp. 1467--1472.

\bibitem{konig2020safety}
C.~K{\"o}nig, M.~Khosravi, M.~Maier, R.~S. Smith, A.~Rupenyan, and J.~Lygeros,
  ``Safety-aware cascade controller tuning using constrained {B}ayesian
  optimization,'' \emph{arXiv preprint arXiv:2010.15211}, 2020.

\bibitem{bansal2017goal}
S.~Bansal, R.~Calandra, T.~Xiao, S.~Levine, and C.~J. Tomlin, ``Goal-driven
  dynamics learning via bayesian optimization,'' in \emph{2017 IEEE 56th Annu.
  Conf. on Decis. and Control (CDC)}.\hskip 1em plus 0.5em minus 0.4em\relax
  IEEE, 2017, pp. 5168--5173.

\bibitem{piga2019performance}
D.~Piga, M.~Forgione, S.~Formentin, and A.~Bemporad, ``Performance-oriented
  model learning for data-driven mpc design,'' \emph{IEEE Contr. Syst. Lett.},
  vol.~3, no.~3, pp. 577--582, 2019.

\bibitem{paulson2020data}
J.~A. Paulson and A.~Mesbah, ``Data-driven scenario optimization for automated
  controller tuning with probabilistic performance guarantees,'' \emph{IEEE
  Contr. Syst. Lett.}, vol.~5, no.~4, pp. 1477--1482, 2020.

\bibitem{baheri2017altitude}
A.~Baheri and C.~Vermillion, ``Altitude optimization of airborne wind energy
  systems: A bayesian optimization approach,'' in \emph{2017 Amer. Control
  Conf. (ACC)}.\hskip 1em plus 0.5em minus 0.4em\relax IEEE, 2017, pp.
  1365--1370.

\bibitem{baheri2020waypoint}
------, ``Waypoint optimization using {B}ayesian optimization: A case study in
  airborne wind energy systems,'' in \emph{2020 Amer. Control Conf.
  (ACC)}.\hskip 1em plus 0.5em minus 0.4em\relax IEEE, 2020, pp. 5102--5017.

\bibitem{pal2020multi}
A.~Pal, L.~Zhu, Y.~Wang, and G.~G. Zhu, ``Multi-objective stochastic {B}ayesian
  optimization for iterative engine calibration,'' in \emph{Proc. of the Amer.
  Control Conf.}\hskip 1em plus 0.5em minus 0.4em\relax IEEE, 2020, pp.
  4893--4898.

\bibitem{sui2015safe}
Y.~Sui, A.~Gotovos, J.~Burdick, and A.~Krause, ``Safe exploration for
  optimization with {G}aussian processes,'' in \emph{Proc. of the Int. Conf. on
  Mach. Learn.}, 2015, pp. 997--1005.

\bibitem{sui2018stagewise}
Y.~Sui, J.~Burdick, Y.~Yue \emph{et~al.}, ``Stage-wise safe {B}ayesian
  optimization with {G}aussian processes,'' in \emph{Proc. of the Int. Conf. on
  Mach. Learn.}, 2018, pp. 4781--4789.

\bibitem{turchetta2020safe}
M.~Turchetta, F.~Berkenkamp, and A.~Krause, ``Safe exploration for interactive
  machine learning,'' \emph{Advances in Neural Inf. Process. Syst. 32}, vol.~4,
  pp. 2868--2878, 2020.

\bibitem{chakrabarty2021safe}
A.~Chakrabarty and M.~Benosman, ``Safe learning-based observers for unknown
  nonlinear systems using {B}ayesian optimization,'' \emph{Automatica}, vol.
  133, p. 109860, 2021.

\bibitem{gardner2014bayesian}
J.~R. Gardner, M.~J. Kusner, Z.~E. Xu, K.~Q. Weinberger, and J.~P. Cunningham,
  ``Bayesian optimization with inequality constraints.'' in \emph{ICML}, vol.
  2014, 2014, pp. 937--945.

\bibitem{gelbart2014bayesian}
M.~A. Gelbart, J.~Snoek, and R.~P. Adams, ``Bayesian optimization with unknown
  constraints,'' in \emph{Proc. of the 30th Conf. on Uncertainty in Artif.
  Intell.}, ser. UAI'14.\hskip 1em plus 0.5em minus 0.4em\relax Arlington,
  Virginia, USA: AUAI Press, 2014, p. 250–259.

\bibitem{snoek2012practical}
J.~Snoek, H.~Larochelle, and R.~P. Adams, ``Practical bayesian optimization of
  machine learning algorithms,'' \emph{Advances in Neural Inf. Process. Syst.},
  vol.~25, 2012.

\bibitem{lee2020cost}
E.~H. Lee, V.~Perrone, C.~Archambeau, and M.~Seeger, ``Cost-aware bayesian
  optimization,'' \emph{arXiv preprint arXiv:2003.10870}, 2020.

\bibitem{lam2016bayesian}
R.~Lam, K.~Willcox, and D.~H. Wolpert, ``Bayesian optimization with a finite
  budget: An approximate dynamic programming approach,'' \emph{Advances in
  Neural Inf. Process. Syst.}, vol.~29, pp. 883--891, 2016.

\bibitem{marco2020excursion}
A.~Marco, A.~von Rohr, D.~Baumann, J.~M. Hern{\'a}ndez-Lobato, and S.~Trimpe,
  ``Excursion search for constrained bayesian optimization under a limited
  budget of failures,'' \emph{arXiv preprint arXiv:2005.07443}, 2020.

\bibitem{burns2018proportional}
D.~J. Burns, C.~R. Laughman, and M.~Guay, ``Proportional-integral extremum
  seeking for vapor compression systems,'' \emph{IEEE Trans. Control Syst.
  Technol.}, vol.~28, no.~2, pp. 403--412, 2018.

\bibitem{modelica2017a}
\BIBentryALTinterwordspacing
{Modelica Association}, ``Modelica specification, {V}ersion 3.4,'' 2017.
  [Online]. Available: \url{www.modelica.org}
\BIBentrySTDinterwordspacing

\bibitem{dymola2020a}
{Dassault Systemes}, ``Dymola 2020,'' 2019.

\bibitem{fmi2019a}
\BIBentryALTinterwordspacing
{Modelica Association}, ``Functional {M}ockup {I}nterface for {M}odel
  {E}xchange and {C}o-{S}imulation, {V}ersion 2.0.1,'' 2019. [Online].
  Available: \url{www.fmi-standard.org}
\BIBentrySTDinterwordspacing

\bibitem{gpy2014}
{GPy}, ``{GPy}: A gaussian process framework in python,''
  \url{http://github.com/SheffieldML/GPy}, since 2012.

\bibitem{berkenkamp2016safe}
F.~Berkenkamp, A.~P. Schoellig, and A.~Krause, ``Safe controller optimization
  for quadrotors with gaussian processes,'' in \emph{2016 IEEE Int. Conf. on
  Robot. and Automat. (ICRA)}.\hskip 1em plus 0.5em minus 0.4em\relax IEEE,
  2016, pp. 491--496.

\end{thebibliography}
\vspace{.5cm}

\end{document}